\documentclass{article}
\usepackage{spconf,amsmath,graphicx}
\usepackage{cite}
\usepackage{algorithmic}
\usepackage{graphicx}
\usepackage{textcomp}
\usepackage{amsmath,amsfonts, amssymb, graphicx,amsthm,epsfig}
\usepackage{bm,bbm} 
\usepackage{dsfont}
\usepackage{mathtools}
\usepackage{array,multirow}
\usepackage{tikz-network}
\usepackage{subfig}
\usepackage{float}
\usepackage{caption}

\newcommand{\bxi}{\bm{\xi}}
\newcommand{\bpsi}{\bm{\psi}}
\newcommand{\bmu}{\bm{\mu}}

\newcommand{\cA}{\mathcal{A}}
\newcommand{\cK}{\mathcal{K}}

\newcommand{\bSa}{\bm{S}}

\newcommand{\blambda}{\bm{\lambda}}
\newcommand{\br}{\bm{r}}

\newcommand{\bL}{\bm{L}}

\newcommand{\bx}{\bm{x}}

\newcommand{\E}{\mathbb{E}}

\newcommand{\dkl}{D_{\textup{KL}}}

\newcommand{\ars}{\textup{A}}
\newcommand{\grs}{\textup{G}}
\newcommand{\vari}{\textup{Var}}
\renewcommand{\Pr}{\mathbb{P}}

\DeclareMathOperator{\T}{\mathsf{T}}

\renewcommand{\qedsymbol}{$\blacksquare$}

\newtheorem{lemma}{Lemma}
\newtheorem{corollary}{Corollary}
\newtheorem{assumption}{Assumption}

\newtheorem{remark}{Remark}
\newtheorem{example}{Example}

\title{On the Fusion Strategies for Federated Decision Making}
\name{Mert Kayaalp$^\star$, Yunus \.Inan$^\star$, Visa Koivunen$^\dagger$, Emre Telatar$^\star$, Ali H. Sayed$^\star$ \thanks{This work was supported in part by grant 205121-184999 from the Swiss National Science Foundation (SNSF). Emails: \{mert.kayaalp, yunus.inan, emre.telatar, ali.sayed\}@epfl.ch., visa.koivunen@aalto.fi}} 
\address{$^\star$École Polytechnique Fédérale de Lausanne (EPFL) \\ $^\dagger$Aalto University}

\begin{document}
\ninept
\maketitle
\begin{abstract}
We consider the problem of information aggregation in federated decision making, where a group of agents collaborate to infer the underlying state of nature without sharing their private data with the central processor or each other. We analyze the non-Bayesian social learning strategy in which agents incorporate their individual observations into their opinions (i.e., soft-decisions) with Bayes rule, and the central processor aggregates these opinions by arithmetic or geometric averaging. Building on our previous work, we establish that both pooling strategies result in asymptotic normality characterization of the system, which, for instance, can be utilized to derive approximate expressions for the error probability. We verify the theoretical findings with simulations and compare both strategies.
\end{abstract}
\section{Introduction and Related Work}

 The centralized strategy has been the traditional form for collaborative decision-making, where data is collected from all sources and processed at a central controller, fusion center or cloud. However, this approach raises concerns about data privacy and communication costs. In this work, we study the paradigm of federated decision-making, which enables cooperative decision-making without data collection \cite{elkordy2023,wang2021federated,ramage2020federated}. For example, in healthcare applications, federated decision-making can be used to test possible hypotheses without gathering private data from siloed institutions \cite{hallock2021}. Moreover, in classification tasks, agents (e.g., mobile device users) may have different views of a physical phenomenon such as scenery as in multi-view learning \cite{bordignon2021learning}, and the server may want to classify the scene or infer the state of environment without collecting the local data of users. 

 The setting we consider in this work, which is described in Sec.~\ref{sec:problem_formulation}, is an instance of locally Bayesian (a.k.a. non-Bayesian) social learning \cite{jadbabaie_2012,zhao_2012,nedic_2017,lalitha_2018,bordignon2021adaptive} that is based on general decentralized networks. Here, we only consider the special case of a star network topology where all agents are connected to a central node. In social learning, the goal is to infer the hypothesis that best describes the observations received from the environment. The true hypothesis is common for all agents and hence cooperation is of interest for all agents. To this end, at each time instant, \((i)\) distributed agents process their private data locally and send their opinions (i.e., soft-decisions) to the server, \((ii)\) the server pools the information received from agents and broadcasts the aggregated belief back to the agents, and \((iii)\) the agents use the broadcasted belief as a prior belief for their subsequent opinions. 
 
 Note that this setting has close ties with the distributed detection literature \cite{tsitsiklis93, varshney2012book, zou2010cooperative,bajovic2012large, icc_detection}, which typically involves sharing log-likelihood ratios for each data and assumes spatial independence and/or homogeneity across agents. In contrast, in the social learning framework, agents share beliefs (posteriors) formed over time, rather than a statistic about individual data. In this way, they can incorporate their individual priors as well. Furthermore, in social learning, data is not required to be independent between agents nor identically distributed. In fact, agents can possess different data types (e.g., text, video), as in multi-modal machine learning --- see Fig.~\ref{fig:federated_illustration}. The federated decision-making is also related to federated learning \cite{mcmahan2017,li2020,kairouz2021advances,rizk2022federated} as they both fall under the umbrella of federated analytics \cite{elkordy2023,wang2021federated,ramage2020federated}, which aims to enable cooperation among multiple users while preserving privacy of data. However, federated learning focuses on training machine learning models with labeled data samples, whereas the current setting involves making collaborative inferences without sharing raw observations.
  \begin{figure}[]
     \centering
     \includegraphics[width=0.7\columnwidth]{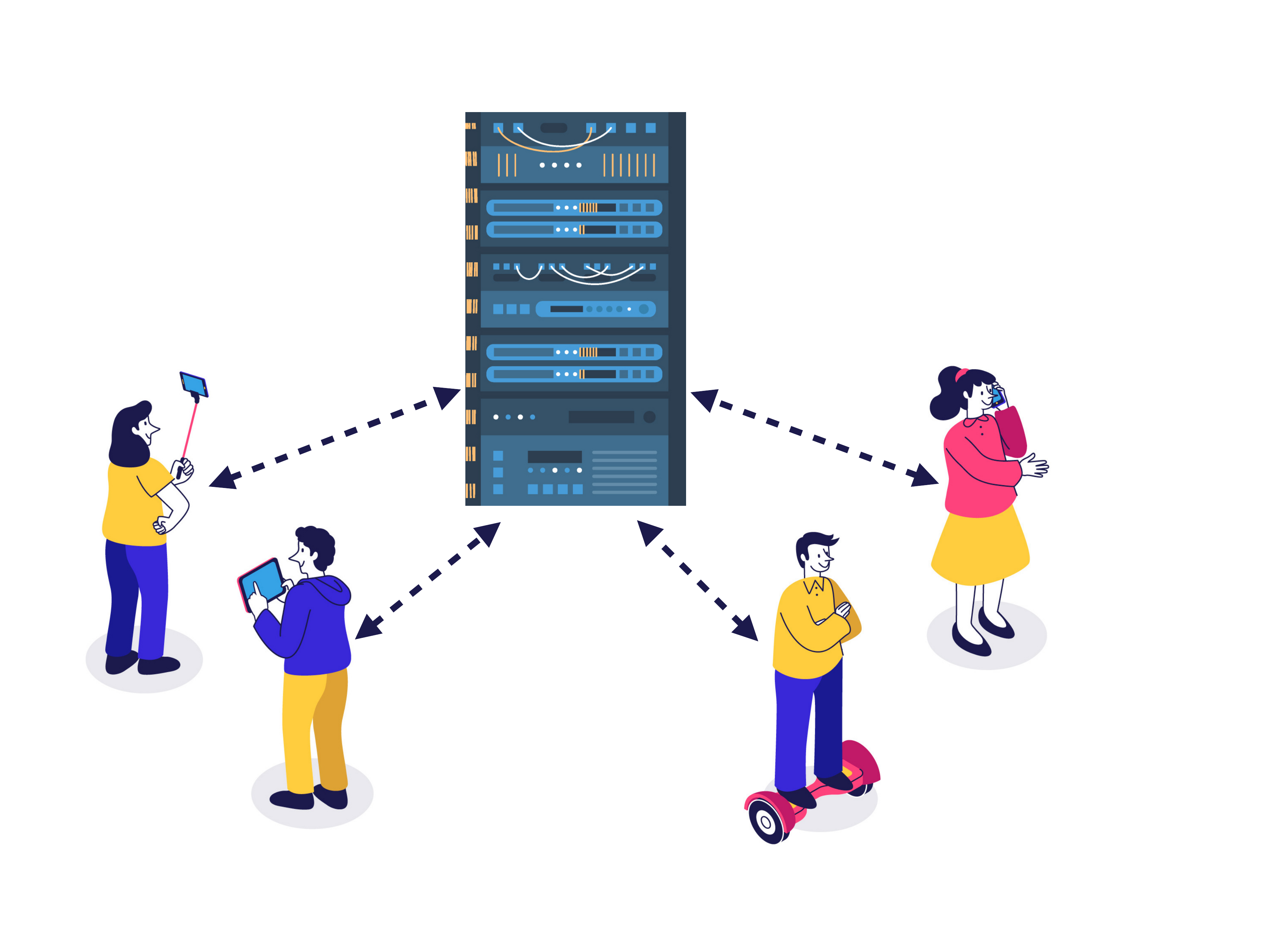}
     \caption[]{Data types at the edge devices can be highly heterogeneous. Image from: freepik.com.}
     \label{fig:federated_illustration}
     \vspace{-1em}
 \end{figure}

 
In this paper, we address the following problem: How should the server aggregate information received from multiple distributed agents that process the data locally and share only their local soft decision statistics? Two of the most widely used methods for combining probability vectors are \emph{arithmetic} averaging (AA), a.k.a. linear opinion pooling, and \emph{geometric} averaging (GA), a.k.a. logarithmic opinion pooling. The distinct properties of these procedures are well-established in the literature --- see, e.g., \cite{koliander2022fusion}. In our previous work \cite{kayaalp2022aaga_journal}, we investigated the performance of AA and GA strategies in the general social learning setting. For the special case of federated architectures, we derived an exact expression for the asymptotic convergence rates of beliefs. Different from prior works that only compare AA and GA with some assumptions, such as assuming only one step of fusion \cite{li2019second,li2021some,gao2020} or assuming specific/particular distributions like Gaussian or Poisson \cite{li2019second,gao2020,chang2010}, our analysis considered repeated fusion of beliefs without any assumption on the distribution of data. In this work, we extend our previous findings and show that the asymptotic behavior of both AA and GA can be characterized with normal distributions and the asymptotic error probabilities of both strategies can be established accordingly. The results provide further insight into the arithmetic and geometric fusion rules, and can be utilized for designing federated decision-making systems. \\

\noindent \textbf{Notation}. Random variables are written in boldface letters. We use ``proportional to'' symbol \(\propto\) whenever the LHS of an equation is a proper normalization of the RHS. The KL-divergence between two distributions \(p\) and \(q\) is denoted by \(\dkl (p||q)\).

\section{Federated Decision Making}\label{sec:problem_formulation}

We consider a distributed hypothesis testing task in which a group \(\mathcal{K}\) of \( K \) agents, aided by a central controller or server, performs inference on a phenomenon of interest. More formally, the agents seek to choose the true hypothesis \( \theta^\circ \) from a finite set of \( H \) hypotheses, \( \Theta=\{1,2,\dots,H \}\). At each time instant \(i\), agent \(k\) receives a personal observation \(\bxi_{k,i}\) from the environment that is distributed according to the marginal likelihood function \(L_k (\cdot | \theta^\circ )\) (since the true hypothesis is $\theta^\circ$). The likelihood functions are allowed to depend on the agent index \(k\), so that the data distributions across the different agents can be heterogeneous. We assume that the environment is stationary, and that the data is independent and identically distributed (i.i.d.) over time at each agent. However, we do not make any assumption on spatial dependence, i.e., there is no assumption on the dependence among different agents. Each agent \(k\) only knows its own local likelihood models \(L_k (\xi | \theta)\) for each \(\theta \in \Theta\), which represent how likely an observation \(\xi\) is produced by \(\theta\). Furthermore, the server does not need to possess any information about the distribution of data over the agents.

The agents do not share their private raw data with the server: all data processing and ownership are local. Each agent \(k\) only shares its belief (i.e., opinion) \( \bpsi_{k,i}\) about the possible hypotheses, which is a probability mass function over \(\Theta\). In other words, \( \bpsi_{k,i} (\theta)\) represents the amount of confidence agent \( k \) has on the proposition ``\( \theta = \theta^\circ \)'', at time instant \(i\). Once the server receives the intermediate beliefs from the distributed agents, it pools them and broadcasts the averaged belief back to the agents. In this work, we study the widely used linear (AA) and logarithmic (GA) pooling strategies, although there are other variations \cite{koliander2022fusion}. We now motivate the setting with two real-world applications.

\begin{example}[\textbf{Crowd size estimation}]
Crowd counting aims to estimate the number of people in a given area using visual or other sensory data. The possible hypotheses are the size of the crowd, such as ``the number of people in the event is between 500-1000''. Agents, i.e., users cooperating with the server, obtain data in the form of photos and videos of their surroundings, or their mobile devices can measure the number of possible connections over Bluetooth. Due to privacy concerns, users may prefer not to share their raw data with the server. In addition, transmitting raw data can create communication bottlenecks at the server. Therefore, employing the social learning strategy, users would only share their beliefs over the set of hypotheses. They can form their beliefs by using local likelihood models, e.g., by using pre-trained neural network models on their devices to estimate the crowd size based on images \cite{liu2019context}. \qed
\end{example}

\begin{example}[\textbf{Flexible spectrum use}]
 Spatially distributed sensors or mobile devices can try to estimate the number of targets or active wireless emitters within a particular area. For example, in the radio spectrum, the mobile devices can measure the local occupancy and field levels to distinguish between idle and congested spectrum, or map the radio environment. This practice can enhance the environmental awareness, and allocate the resources optimally in order to improve the wireless communication. Exchanging beliefs only rather than raw data can be advantageous in terms of privacy and communication costs.
 \qed
\end{example}

To form the beliefs, agents execute the non-Bayesian social learning strategy \cite{jadbabaie_2012,zhao_2012,nedic_2017,lalitha_2018} under the special case of star topologies. At each time instant \(i\), agents first process their observations in a \emph{locally} Bayesian manner:
\begin{align}\label{eq:dif_adapt_step}
 \bpsi_{k,i} (\theta) &\propto L_k(\bxi_{k,i} | \theta)\bmu_{i-1} (\theta) \qquad \text{(Adapt)},
\end{align}
and then share these intermediate beliefs with the server. If the server employs AA for information fusion, the updated belief is a weighted arithmetic average of the intermediate beliefs from different agents: 
\begin{align}\label{eq:linear_fusion}
   \bmu_{i}(\theta) &= \sum_{k \in \mathcal{K}}{\pi_k} \bpsi_{k,i}(\theta) \qquad\text{(AA)}.
\end{align}
Here, \(\pi= [\pi_1, \dots, \pi_K]^{\T}\) is a vector of the confidence weights \(\pi_k\) the server assigns to each agent \(k\) \cite{varshney2012book,Sayed14}. We assume that these weights are positive, sum up to 1, and are constant over time. Alternatively, if the fusion rule is GA, the updated belief is a weighted geometric average:
\begin{align}\label{eq:geometric_fusion}
   \bmu_{i}(\theta) &\propto  \prod_{k \in \mathcal{K} } (\bpsi_{k,i} (\theta ))^{\pi_k} \quad\text{(GA)}.
\end{align}
The server then sends the updated belief to the agents, and they execute the same steps repeatedly over time by using the updated beliefs and locally observed data. 

\begin{remark}
An important distinction between AA and GA lies in their handling of initial beliefs and their sensitivity to individual beliefs. In AA, only one agent having a positive initial belief on the true hypothesis is sufficient. In contrast, in GA, all agents must have positive initial beliefs on the truth in order not to discard it from the beginning. Furthermore, GA is more sensitive to small beliefs compared to AA. This is because, if some agents' beliefs are very small, they impact the product more significantly compared to the sum. In other words, GA grants the agents a veto power: if one agent transmits a belief entry as 0, the combined belief also becomes 0. In general, giving this level of autonomy to individual agents depends on the application at hand. \qed
\end{remark}

\section{Asymptotic Normality of AA and GA}

In this section, we analyze the asymptotic behavior under both fusion strategies. In order to avoid pathological cases, we assume that each observation has finite information about the true hypothesis. That is to say, for each agent \( k \) and for each hypothesis \( \theta \in \Theta \), we assume \( \dkl( L_k (\cdot | \theta^\circ ) ||  L_k (\cdot | \theta )) < \infty \). This implies that likelihood functions have the same support. Moreover, in order to uniquely distinguish the true hypothesis, we need the following condition.

\begin{assumption}[\textbf{Global identifiability}]\label{assum:global_identification}
For each wrong hypothesis \( \theta \neq \theta^\circ \), there exists at least one clear-sighted agent \( k \) with \(  \dkl( L_k (\cdot | \theta^\circ ) ||  L_k (\cdot | \theta )) > 0 \). \hfill\qedsymbol
\end{assumption}
Note that this assumption does not require \emph{local} identifiability, which is an agent's ability of inferring \( \theta^\circ \) without any cooperation. As a result, all agents can benefit from cooperation. Under Assumption~\ref{assum:global_identification}, it is already known that AA \cite{jadbabaie_2012,zhao_2012,jadbabaie2013information} and GA \cite{nedic_2017,lalitha_2018} result in consistent truth learning, that is, the beliefs on the true hypothesis converge to one almost surely. In this work, we further show that under both AA and GA fusion rules, the asymptotic error probabilities can be calculated with Gaussian cumulative distribution functions (CDFs). We start with the AA strategy.

\begin{lemma}[\textbf{Asymptotic normality of AA}]\label{prop:gaussianAA} For each wrong hypothesis $\theta \neq \theta^\circ$, we define the mean and variance
    \begin{equation}\label{eq:aa_mean}
    \rho_{\ars} \triangleq -\E\Bigg[\log\Bigg(\sum_{k \in \cK} \pi_k \br_{k,i}\Bigg)\Bigg]
    ,\ \sigma_{\ars}^2 \triangleq \vari\Bigg[\log\Bigg(\sum_{k\in\cK}\pi_k \br_{k,i}\Bigg)\Bigg],
    \end{equation}
    in terms of the likelihood ratios \vspace{-0.7em}
    \begin{equation}
        \br_{k,i} \triangleq \frac{L_k (\bxi_{k,i} | \theta)}{ L_k (\bxi_{k,i} | \theta^\circ)}.
    \end{equation}
    If $\sigma_{\ars}^2 < \infty$, it holds under Assumption~\ref{assum:global_identification} that
    \begin{equation}
        \lim_{i \to \infty} \Pr\bigg(\frac{\log\bmu_i(\theta)+\rho_{\ars} i}{\sigma_{\ars}\sqrt{i}} \leq t\bigg) = \Phi(t),
    \end{equation}
    where $\Phi(t)$ is the standard Gaussian cumulative distribution function.
\end{lemma}
\begin{proof}[Proof Sketch] For the complete proof, see Appendix~\ref{appendix:prop_aa}. Consider an arbitrary $\theta \neq \theta^\circ$. For notational simplicity, define
\begin{equation}
 \bx_i \triangleq \log {\bmu_i(\theta)}, \quad \bL_i \triangleq \log \Bigg (\sum_{k \in \cK} \pi_k \br_{k,i} \Bigg).
\end{equation}
Since $\bL_i$'s are assumed to be i.i.d. and of finite variance, the central limit theorem \cite{sayed_2022} yields:
\begin{equation}
 \lim_{i \to \infty} \Pr\bigg( \dfrac{\sum_{j=1}^i \bL_j + \rho_\ars i}{\sigma_\ars\sqrt{i}} \leq t\bigg) = \Phi(t).
\end{equation}
Therefore, if we show that the distance between $\frac 1 {\sqrt{i}} \bx_i$ and $ \frac 1 {\sqrt{i}} \sum_{j = 1}^i\bL_j$ goes to zero with high probability, we are done. To that end, we utilize the result  from \cite{kayaalp2022aaga_journal} which states
\begin{equation}\label{eq:from_journal_almost_sure}
    \frac 1 i \bx_i \to -\rho_\ars  \quad \text{as} \:\:  i \to \infty ,\:\: \text{almost surely}.
\end{equation}
In particular, we use the following bootstrapping argument on \eqref{eq:from_journal_almost_sure}: There exists an \(i_0\) after which (i.e., \(\forall i \geq i_0\))
\begin{equation}
    \bmu_i(\theta) \leq \exp\{{-i(\rho_{\ars} - \epsilon)}\},\ \forall \theta \neq \theta^\circ
\end{equation}
with high probability. This result enables a finer study of the evolution of \(\bx_i\) as \(i \to \infty\).
\end{proof}
\noindent Note that Lemma~\ref{prop:gaussianAA} can equivalently be given in terms of the log-belief ratios.
\begin{corollary} If we define the log-belief ratio $\blambda_i(\theta) \triangleq \log\dfrac{\bmu_i(\theta^\circ)}{\bmu_i(\theta)}$, for each $\theta \neq \theta^\circ$, Lemma~\ref{prop:gaussianAA} implies that
    \begin{equation}
        \lim_{i \to \infty} \Pr\bigg(\frac{\blambda_i(\theta)-\rho_A i}{\sigma_A\sqrt{i}} \leq t\bigg) = \Phi(t).
    \end{equation}
\end{corollary}
\begin{proof}[Proof Sketch]
For the complete proof, see Appendix~\ref{appendix:aa_corollary}. Essentially, the proof follows from consistent truth learning, i.e., \(\bmu_i(\theta^\circ)~\to~1\) almost surely as \(i \to \infty\). This implies that, as \(i \to \infty\), \(\log {\bmu_i(\theta)}\) and \(\log {\bmu_i(\theta)}-\log {\bmu_i(\theta^\circ)}\) behave similarly.
\end{proof}
Next, we investigate the behavior of beliefs under the GA fusion rule.
\begin{lemma}[\textbf{Asymptotic normality of GA}]\label{prop:ga_normality} For each wrong hypothesis $\theta \neq \theta^\circ$, we define
the mean
    \begin{equation}\label{eq:ga_mean}
    \!\!\rho_{\grs} \triangleq -\E\Bigg[\sum_{k \in \cK} \pi_k \log \br_{k,i} \Bigg] \!\!=\!\!\sum_{k \in \cK} \!\pi_k \dkl(L_k(.|\theta^{\circ}) || L_k(.|\theta))
    \end{equation}
    and the variance \vspace{-0.7em}
    \begin{equation}
    \sigma_{\grs}^2 \triangleq \vari\Bigg[\sum_{k \in \cK} \pi_k \log \br_{k,i}\Bigg].\vspace{-0.5em}
    \end{equation}
\vspace{-0.5em}If $\sigma_\grs^2 < \infty$, then under Assumption~\ref{assum:global_identification} it holds that
    \begin{equation}
        \lim_{i \to \infty} \Pr\bigg(\frac{\blambda_i(\theta)-\rho_{\grs} i}{\sigma_{\grs}\sqrt{i}} \leq t\bigg) = \Phi(t).
    \end{equation}
\end{lemma}
\begin{proof}
The proofs from \cite{nedic_2017,lalitha_2018} for the convergence rate analysis can be easily extended to establish asymptotic normality. From \eqref{eq:dif_adapt_step} and \eqref{eq:geometric_fusion}, the log-belief ratio $\blambda_i(\theta)$ evolves according to the recursion: 
\begin{equation}
    \blambda_i(\theta) = \blambda_{i-1}(\theta) - \sum_{k \in \cK} \pi_k \log \br_{k,i}.
\end{equation}
Expanding the recursion over time, we obtain
\begin{equation}
    \blambda_i(\theta) = -\sum_{j=1}^{i} \bigg(\sum_{k \in \cK} \pi_k \log \br_{k,j}\bigg) + \blambda_0(\theta).
\end{equation}
Since \((i)\) $\blambda_0(\theta) < \infty$ due to initial beliefs being positive, \((ii)\) $\br_{k,i}$'s are i.i.d., and \((iii)\) $\sigma_\grs^2 < \infty$, an application of the central limit theorem \cite{sayed_2022} concludes the proof.
\end{proof}
\begin{figure*}[h!]\vspace*{-20pt}
	\centering
	\subfloat[a][]{
		\includegraphics[width=0.40\linewidth]{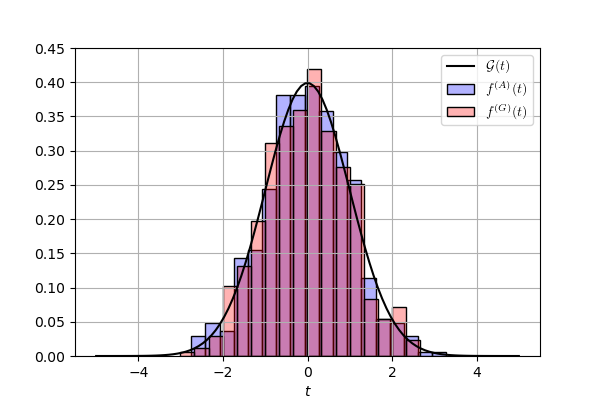}\label{fig:histo1}
	}\hfil
	\subfloat[b][]{
		\includegraphics[width=0.40\linewidth]{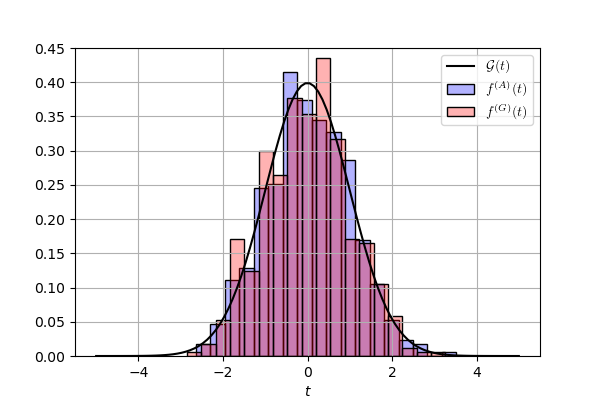}\label{fig:histo2}
	}\hfil
 \vspace*{-10pt}
	\caption{(a) For the first experiment with Gaussian data, the histograms of $\widetilde{\blambda}^{(A)}_i(\theta)$ and $\widetilde{\blambda}^{(G)}_i(\theta)$ are drawn as shaded bar plots with colors blue and red respectively. The standard normal density $\mathcal{G}(t)$ is shown as solid black line.
 (b) For the second experiment with exponentially distributed data, the histograms of $\widetilde{\blambda}^{(A)}_i(\theta)$ and $\widetilde{\blambda}^{(G)}_i(\theta)$ are drawn as shaded bar plots with colors blue and red respectively. The standard normal density $\mathcal{G}(t)$ is shown as solid black line.}
 \vspace*{-24pt}
\end{figure*}

\begin{figure*}[h!]
	\centering
	\subfloat[a][]{
		\includegraphics[width=0.40\linewidth]{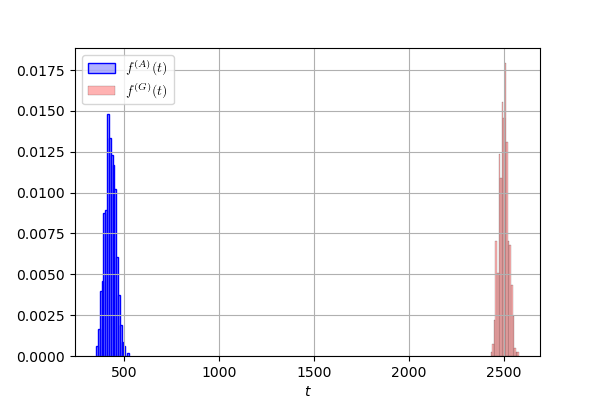}\label{fig:indep}
	}\hfil
	\subfloat[b][]{
		\includegraphics[width=0.40\linewidth]{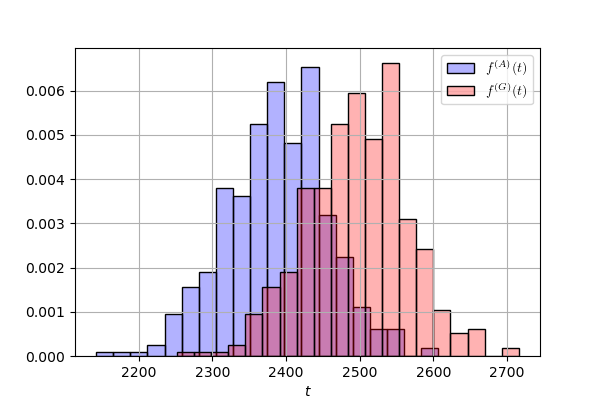}\label{fig:corr}
	}\hfil
	\vspace*{-10pt}\caption{(a) For the uncorrelated Gaussian case, histograms of ${\blambda}^{(A)}_i(\theta)$ and ${\blambda}^{(G)}_i(\theta)$ are drawn as shaded bar plots with colors blue and red respectively.
 (b) For the correlated Gaussian case, histograms of ${\blambda}^{(A)}_i(\theta)$ and ${\blambda}^{(G)}_i(\theta)$ are drawn as shaded bar plots with colors blue and red respectively.}
\end{figure*}

In Lemmas~\ref{prop:gaussianAA} and \ref{prop:ga_normality}, we have established the asymptotic normality for each wrong hypothesis, which allows approximating the error probabilities when agents or the central processor employs a maximum a-posteriori (MAP) rule to estimate \(\theta^\circ\). It is important to note that under spatial independence, GA matches the optimal Bayes posterior, which attains minimum error probability, when the confidence weight \(\pi_k\) on each agent \(k\) is uniform. Additionally, the convergence rates \(\rho_{\grs}\) and \(\rho_{\ars}\) were already established in previous works \cite{nedic_2017,lalitha_2018} and \cite{kayaalp2022aaga_journal} respectively. Notice that \(\rho_{\grs} \geq \rho_{\ars}\) --- applying Jensen's inequality to the equation on the left in \eqref{eq:aa_mean}, we obtain \eqref{eq:ga_mean}. However, this does not directly imply that GA has a better performance in probability of error criterion since it depends on the variances as well. Our simulations indicate, however, that GA has lower error probability in general. We leave it to future work to determine whether GA is superior to AA in terms of asymptotic error probability under any data distribution or confidence weight configuration --- such as in the convergence rate case.

\section{Numerical Results}\label{sec:numerical}

We consider a binary hypothesis testing problem with \(K=10\) distributed agents connected to a fusion center. In the first experiment, agents receive i.i.d. Gaussian observations with mean 0 under the true hypothesis $\theta^\circ$, and with mean 1 under $\theta \neq \theta^\circ$. For both hypotheses, the correlation matrix is the identity matrix $I_K$, i.e., data is independent across agents. A possible application of this experiment is in flexible spectrum use, for example in cognitive radios. The possible hypotheses are whether the spectrum is idle or not. If there is noise only, the signal distributions would be central (i.e., 0 mean). Otherwise, if there is a signal present, it would not be 0 mean. We choose the combination vector as \vspace{-0.5em}
\begin{equation}\pi=[0.13,0.2,0.09,0.15,0.08,0.05,0.1,0.05,0.1,0.05].\vspace{-0.1em}\end{equation}
If we take the time horizon $i = 5000$ and repeat the experiment for 500 realizations, we obtain the histograms given in Figure \ref{fig:histo1}. Specifically, we obtain the normalized log-belief ratios: \begin{equation}
\widetilde{\blambda}^{(A)}_i(\theta) \triangleq \frac{\blambda^{(A)}_i(\theta) - \rho_\ars i}
{\sqrt{i}\sigma_\ars},\quad \widetilde{\blambda}^{(G)}_i(\theta) \triangleq \frac{\blambda^{(G)}_i(\theta) - \rho_\grs i}
{\sqrt{i}\sigma_\grs}
\end{equation}
and plot their histograms in density forms (denoted by $f^{(A)}(t)$ and $f^{(G)}(t)$, respectively) to compare them with the standard normal density $\mathcal{G}(t) \triangleq \frac 1 {\sqrt{2\pi}}\exp(-t^2/2)$. In addition to the visual matching supporting our claims, we performed a Shapiro-Wilk test for testing the normality of $\widetilde{\blambda}^{(A)}_i(\theta)$ and $\widetilde{\blambda}^{(G)}_i(\theta)$. The resulting $p$-values turn out to be 0.66 and 0.54, which strongly support the null hypothesis of normal distribution.

For the second experiment, we assume that agents receive i.i.d. exponential observations with mean 1 under $\theta^\circ$ and with mean $0.5$ under $\theta\neq \theta^\circ$. The other problem parameters are the same as the first experiment. The resulting histograms are given in Figure \ref{fig:histo2}. The Shapiro-Wilk test yields $p$-values 0.93 and 0.53, respectively, again, strongly supporting the null hypothesis of normal distribution.

We also conducted a third experiment to verify that correlations between the agents affect $\rho_{\ars}$ but do not affect $\rho_{\grs}$ (see \eqref{eq:ga_mean}). To compare correlated versus uncorrelated data across agents, we simulate the setting from the first experiment, except that this time we set the correlation matrix as $\Sigma = 0.95 \mathbbm{1}_K \mathbbm{1}^{\T}_K +0.05 I_K$. Note that this choice makes the data highly correlated. Then, we plot the histograms of $\blambda_i^{(A)}(\theta)$ and $\blambda_i^{(G)}(\theta)$ (without normalization) for both uncorrelated and correlated cases in Figure \ref{fig:indep} and \ref{fig:corr}. We can observe that $\rho_{\grs} = 0.5$ is unchanged by the correlations --- the observation that $f^{(G)}$'s are concentrated around $i\times\rho_{\grs} = 2500$ supports this fact. In contrast, $\rho_{\ars}$ increases drastically under the presence of positive correlations. Nevertheless, in both cases considered, GA outperforms AA in terms of the error probability. 

\section{Conclusion}

In this work, we considered two pooling algorithms for federated decision making, namely AA and GA. The choice between AA and GA may depend on the specific task at hand. We study their performance in terms of error probability, and establish asymptotic normality which is a valuable tool that can help avoid the need to simulate the algorithm for large time horizons. Note that the parameters in the normal approximations can be obtained either analytically or through Monte Carlo simulations, which is more efficient compared to the exact simulation of the entire system. We leave the study for asymptotic normality of general decentralized networks to a future work.


\bibliographystyle{IEEEtran}
\bibliography{ref.bib}

\newpage

\appendix

\section{Proof of Lemma 1}\label{appendix:prop_aa}
Consider an arbitrary $\theta \neq \theta^\circ$. We denote the log-belief on $\theta$ as $ \bx_i \triangleq \log {\bmu_i(\theta)}$. It was shown in \cite{kayaalp2022aaga_journal} that $\dfrac 1 i \bx_i \to -\rho_{\ars}$ almost surely as $i \to \infty$. This implies that for any $\epsilon > 0$, there exists an $i_0$ such that
\begin{equation}\label{eq:a0}
    \cA_0 \triangleq \{\bmu_i(\theta) \leq \exp\{{-i(\rho_{\ars} - \epsilon)}\},\ \forall i \geq i_0, \forall \theta \neq \theta^\circ\} 
\end{equation}
is a high-probability event. After this observation, we split the proof into two parts:
\begin{itemize}
    \item[(a)] $\liminf_{i \to \infty} \Pr\bigg(\dfrac{\bx_i+\rho_{\ars} i}{\sigma_{\ars}\sqrt{i}} < t\bigg) \geq \Phi(t)$
    \item[(b)] $\limsup_{i \to \infty} \Pr\bigg(\dfrac{\bx_i+\rho_{\ars} i}{\sigma_{\ars}\sqrt{i}} \leq t\bigg) \leq \Phi(t)$.
\end{itemize}
Then, it readily follows from (a) and (b) that
\begin{equation}
    \lim_{i \to \infty} \Pr\bigg(\frac{\bx_i+\rho_{\ars} i}{\sigma_{\ars}\sqrt{i}} \leq t\bigg) = \Phi(t).
\end{equation}
\subsection{Proof of (a)}

Using \eqref{eq:dif_adapt_step} and \eqref{eq:linear_fusion}, the log-belief $\bx_i$ can be upper bounded under the event $\cA_0$ as
\begin{equation}\label{eq:iteration_upperbnd}
    \bx_i \leq \bx_{i_0} + \!\!\!\sum_{j=i_0+1}^i \bL_j \!- \!\!\!\sum_{j=i_0}^{i-1}\log(1 - H \exp\{{-j(\rho_{\ars} - \epsilon)}\})
\end{equation}
where we defined $\bL_i \triangleq \log(\sum_k \pi_k \br_{k,i})$ and recall that $H$ is the number of hypotheses. Since
\begin{equation}
-\sum_{j=i_0}^\infty\log(1 - H\exp \{{-j(\rho_{\ars} - \epsilon)}) \} < \infty,
\end{equation}
for every $\delta > 0$ there must exist an $i_1^+ > i_0$ such that for all $i \geq i_1^+$:
\begin{equation}
    -\frac 1 {\sqrt{i}} \sum_{j=i_0}^{i-1}\log(1 -  H\exp \{{-j(\rho_{\ars} - \epsilon)}) \}) \leq \delta.
\end{equation}
Consequently, the event
\begin{equation}
\cA_1^{+} \triangleq \Big \{\frac 1 {\sqrt{i}} \bx_i  \leq  \frac 1 {\sqrt{i}} \bx_{i_0} + \frac 1 {\sqrt{i}} \sum_{j=i_0+1}^i \bL_j + \delta, \forall i \geq i_1^+ \Big \}
\end{equation}
is also a high-probability event. Furthermore, since $\bx_{i_0}$ is almost surely finite, there is an $i_2^+ > i_0$ such that
\begin{equation}
    \cA_2^+ \triangleq \{\frac 1 {\sqrt{i}}  \bx_{i_0} \leq \delta, \forall i\geq i_2^+\}
\end{equation}
is a high-probability event as well. Accordingly, for any outcome in $\cA^+ \triangleq \cA_1^+ \cap \cA_2^+$, and for $i \geq \max\{i_1^+,i_2^+\}$, we have 
\begin{equation}
    \frac 1 {\sqrt{i}} \bx_i  \leq  \frac 1 {\sqrt{i}} \sum_{j=i_0+1}^i \bL_j + 2\delta,
\end{equation}
which directly implies
\begin{equation}\label{eq:upperbnd}
    \frac{\bx_i+\rho_{\ars} i}{\sigma_{\ars}\sqrt{i}}  \leq  \frac{\sum_{j=i_0+1}^i \bL_j+\rho_{\ars} i}{\sigma_{\ars}\sqrt{i}} + \frac{2\delta}{\sigma_\ars}.
\end{equation}
Finally, with $\overline{\cA^+}$ denoting the complement of $\cA^+$, we obtain that for $i \geq \max\{i_1^+, i_2^+\}$:
\begin{align}
    &\Pr\bigg(\frac{\bx_i+\rho_{\ars} i}{\sigma_{\ars}\sqrt{i}} \geq t\bigg) \notag \\
    &= \Pr\bigg(\frac{\bx_i+\rho_{\ars} i}{\sigma_{\ars}\sqrt{i}} \geq t \: \cap \: \cA^+\bigg) + 
    \Pr\bigg(\frac{\bx_i+\rho_{\ars} i}{\sigma_{\ars}\sqrt{i}} \geq t \: \cap \: \overline{\cA^+}\bigg) \notag \\
    &\leq \Pr\bigg(\frac{\bx_i+\rho_{\ars} i}{\sigma_{\ars}\sqrt{i}} \geq t \: \cap \: \cA^+\bigg) + 
    \Pr(\overline{\cA^+}) \notag \\
    &\leq \Pr\bigg(\frac{\sum_{j=i_0+1}^i \bL_j+\rho_{\ars} i}{\sigma_{\ars}\sqrt{i}} \geq t - \frac{2\delta}{\sigma_{\ars}}\bigg) + \Pr(\overline{\cA^+}),\label{eq:ccdf}
\end{align}
where \eqref{eq:ccdf} follows from \eqref{eq:upperbnd}. Since $\bL_i$'s are i.i.d. random variables with finite variance, we can use the central limit theorem (CLT) \cite{sayed_2022} and arrive at
\begin{equation}\label{eq:liminf}
    \limsup_{i \to \infty} \Pr\bigg(\frac{\bx_i+\rho_{\ars} i}{\sigma_{\ars}\sqrt{i}} \geq t\bigg) \leq 1 - \Phi(t)
\end{equation}
by using the facts that $\delta$ and $\Pr(\overline{\cA^+})$ can be made arbitrarily small and $\Phi(\cdot)$ is continuous. Note that \eqref{eq:liminf} is equivalent to
\begin{equation}\label{eq:liminf2}
    \liminf_{i \to \infty} \Pr\bigg(\frac{\bx_i+\rho_{\ars} i}{\sigma_{\ars}\sqrt{i}} < t\bigg) \geq \Phi(t).
\end{equation}
\subsection{Proof of (b)}
Similar to the proof of (a), we obtain a lower bound for $\bx_i$ under the event $\cA_0$ by
\begin{align}
    &\!\!\bx_i \geq \bx_{i_0} + \!\!\!\sum_{j=i_0+1}^i \!\! \log\bigg(\sum_{k \in \cK} \frac{\pi_k \br_{k,j}} {1 + \exp\{-(j-1)(\rho_{\ars} - \epsilon)\}\sum_{\theta \neq \theta^\circ} \br_{k,j}(\theta)}\bigg) \notag\\
   &\!\!\! \geq \bx_{i_0} \!\!+ \!\!\!\! \sum_{j=i_0+1}^i \!\!\!\log\!\bigg(\!\sum_{k \in \cK} \frac{\pi_k \br_{k,j}} {1 + \exp\{-(j-1)(\rho_{\ars} - \epsilon)\}\sum_{k \in \cK}\sum_{\theta \neq \theta^\circ} \br_{k,j}(\theta)}\!\bigg) \notag \\
    &\!\!\geq\bx_{i_0}\!\!+\!\!\!\!\!\sum_{j=i_0+1}^i\!\!\!\bL_j
    -\!\!\!\!\!\sum_{j=i_0+1}^i\!\log\bigg({1 + \exp\{-(j-1)(\rho_{\ars} - \epsilon)\}\sum_{k \in \cK}\sum_{\theta \neq \theta^\circ} \br_{k,j}(\theta)}\bigg) \notag \\
    &\!\!\geq \!\bx_{i_0} \!\! +\!\!\!\!\! \sum_{j=i_0+1}^i \!\!\!\bL_j -\!\!\!\!\!\sum_{j=i_0+1}^i \!\!\!\!\exp\{-(j-1)(\rho_{\ars} - \epsilon)\}\sum_{k \in \cK}\sum_{\theta \neq \theta^\circ} \!\!\br_{k,j}(\theta).\label{eq:iteration_lowerbnd}
\end{align}
Comparing with \eqref{eq:iteration_upperbnd}, we aim to show that the rightmost term in \eqref{eq:iteration_lowerbnd}, i.e., 
\begin{equation}
    \bSa_i \triangleq  \sum_{j=i_0+1}^i \exp\{-(j-1)(\rho_{\ars} - \epsilon)\}\sum_{k \in \cK}\sum_{\theta \neq \theta^\circ} \br_{k,j}(\theta)
\end{equation}
tends to zero almost surely when divided by $\sqrt{i}$. Observe that $\bSa_i$ is non-decreasing. Hence, 
\begin{equation}
    \bSa_{\infty} \triangleq \lim_{i\to \infty} \bSa_i
\end{equation}
exists almost surely, i.e., 
\begin{equation}
    \limsup_{i\to \infty} \bSa_i = \liminf_{i\to \infty} \bSa_i.
\end{equation}
Moreover, by the monotone convergence theorem \cite{williams_1991},
\begin{equation}
    \mathbb{E}[\bSa_{\infty}] = \lim_{i \to \infty} \mathbb{E}[\bSa_i] < \infty.
\end{equation}
 Consequently, $\bSa_{\infty}$ must be finite almost surely. Then, by combining the inequality
\begin{equation}
    0 \leq \frac {\bSa_{i}} {\sqrt{i}} \leq \frac {\bSa_{\infty}} {\sqrt{i}}
\end{equation}
and the limit $\dfrac {\bSa_{\infty}} {\sqrt{i}} \to 0$ as \(i \to \infty\) (both of which hold almost surely), we conclude that $\dfrac {\bSa_{i}} {\sqrt{i}} \to 0$ almost surely as well. Therefore, for every $\delta > 0$ there must exist an $i_1^- > i_0$ such that for all $i \geq i_1^-$:
\begin{equation}
        \frac{1}{\sqrt{i}}  \sum_{j=i_0+1}^i \exp\{-(j-1)(\rho_{\ars} - \epsilon)\}\sum_{k \in \cK}\sum_{\theta \neq \theta^\circ} \br_{k,j}(\theta) \leq \delta.
\end{equation}

Then, proceeding similarly to the part (a), we define the high-probability sets: 
\begin{equation}
\cA_1^{-} \triangleq \Big \{\frac 1 {\sqrt{i}} \bx_i  \geq  \frac 1 {\sqrt{i}} \bx_{i_0} + \frac 1 {\sqrt{i}} \sum_{j=i_0+1}^i \bL_j - \delta, \forall i \geq i_1^- \Big \}
\end{equation}
and
\begin{equation}
    \cA_2^- \triangleq \{\frac 1 {\sqrt{i}}  \bx_{i_0} \geq -\delta, \forall i\geq i_2^-\}.
\end{equation}
By defining ${\cA^- \triangleq \cA_1^- \cap \cA_2^-}$, we arrive at the following analog to \eqref{eq:ccdf} for $i \geq \max\{i_1^-,i_2^-\}$:
\begin{align}
    \!\!\!\Pr\bigg(\frac{\bx_i+\rho_{\ars} i}{\sigma_{\ars}\sqrt{i}} \leq t\bigg) \leq \Pr\bigg(\frac{\sum_{j=i_0+1}^i \bL_j+\rho_{\ars} i}{\sigma_{\ars}\sqrt{i}} \leq t + \frac{2\delta}{\sigma_{\ars}}\bigg) + \Pr(\overline{\cA^-}),
\end{align}
which implies
\begin{equation}\label{eq:limsup}
    \limsup_{i \to \infty} \Pr\bigg(\frac{\bx_i+\rho_{\ars} i}{\sigma_{\ars}\sqrt{i}} \leq t\bigg) \leq \Phi(t),
\end{equation}
and the proof is complete.

\section{Proof of Corollary 1}\label{appendix:aa_corollary}

Notice that for any outcome in $\cA_0$, defined in \eqref{eq:a0}, it is true that
\begin{equation}
1-H \exp \{-i(\rho_{\ars}-\epsilon)\} \leq \bmu_i(\theta^\circ) \leq 1,
\end{equation}
which implies
\begin{equation}
\log(1-H\exp\{-i(\rho_{\ars}-\epsilon)\}) \leq \log\bmu_i(\theta^\circ) \leq 0.
\end{equation}
Since 
\begin{equation}
    \dfrac{1}{\sqrt{i}}\log(1-H\exp\{-i(\rho_{\ars}-\epsilon)\}) \to 0,
\end{equation}
the result of Lemma~\ref{prop:gaussianAA} applies here as well.

\end{document}